\documentclass{article}

\usepackage{hyperref}
\usepackage{enumitem}

\setlist{leftmargin=*}

\usepackage[accepted]{icml2020}

\usepackage[utf8]{inputenc}            
\usepackage[T1]{fontenc}               
\usepackage{hyperref}       
\usepackage{url}                       
\usepackage{booktabs}                  
\usepackage{amsfonts}                  
\usepackage{nicefrac}                  
\usepackage{microtype}                 
\usepackage{mystyle}

\usepackage{wrapfig}
\usepackage{subfig}

\usepackage{algorithm}
\usepackage{algorithmic}

\renewcommand{\algorithmicrequire}{\textbf{Input:}}
\renewcommand{\algorithmicensure}{\textbf{Output:}}

\renewcommand{\epsilon}{\varepsilon}

\icmltitlerunning{Wasserstein Measure Coresets}

\begin{document}

\twocolumn[
\icmltitle{Wasserstein Measure Coresets}



\icmlsetsymbol{equal}{*}

\begin{icmlauthorlist}
\icmlauthor{Sebastian Claici}{mit}
\icmlauthor{Aude Genevay}{mit}
\icmlauthor{Justin Solomon}{mit}
\end{icmlauthorlist}

\icmlaffiliation{mit}{Computer Science and Artificial Intelligence Laboratory, Massachusetts Institute of Technology, Cambridge, MA, USA}

\icmlcorrespondingauthor{Sebastian Claici}{sclaici@mit.edu}


\vskip 0.3in
]



\printAffiliationsAndNotice{}  

\begin{abstract}
  The proliferation of large data sets and Bayesian inference techniques motivates demand for better data sparsification. Coresets provide a principled way of summarizing a large dataset via a smaller one that is guaranteed to match the performance of the full data set on specific problems. Classical coresets, however, neglect the underlying data distribution, which is often continuous. We address this oversight by introducing \emph{Wasserstein measure coresets}, an extension of coresets which by definition takes into account generalization. Our formulation of the problem, which essentially consists in minimizing the Wasserstein distance, is solvable via stochastic gradient descent. This yields an algorithm which simply requires sample access to the data distribution and is able to handle large data streams in an online manner. We validate our construction for inference and clustering.
\end{abstract}

\begin{figure}
    \centering
    \begin{tabular}{@{}c@{\hspace{.1in}}c@{}}
\includegraphics[width=.46\linewidth]{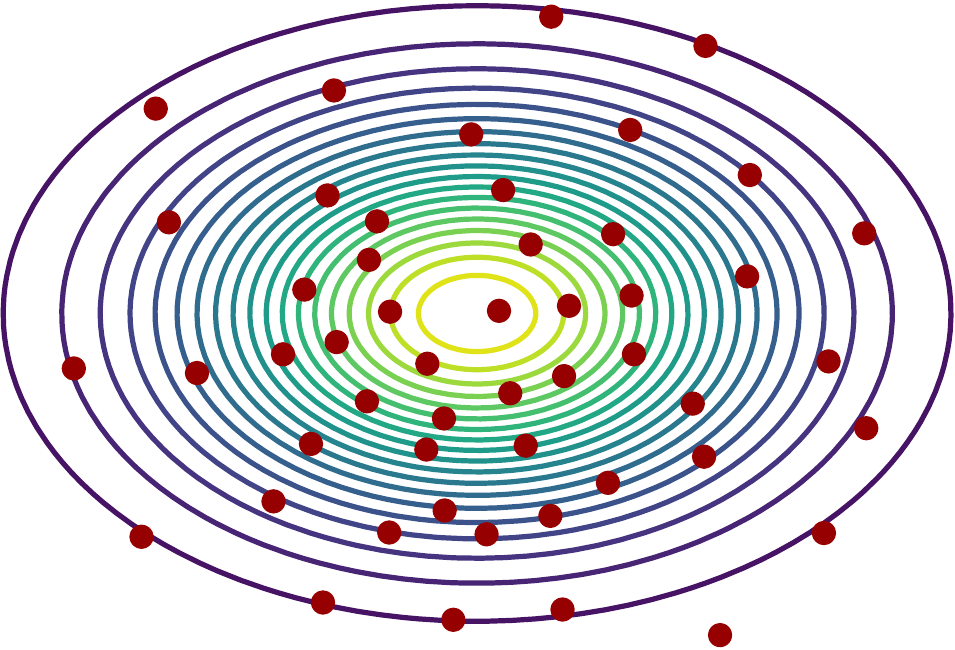}&
\includegraphics[width=.46\linewidth]{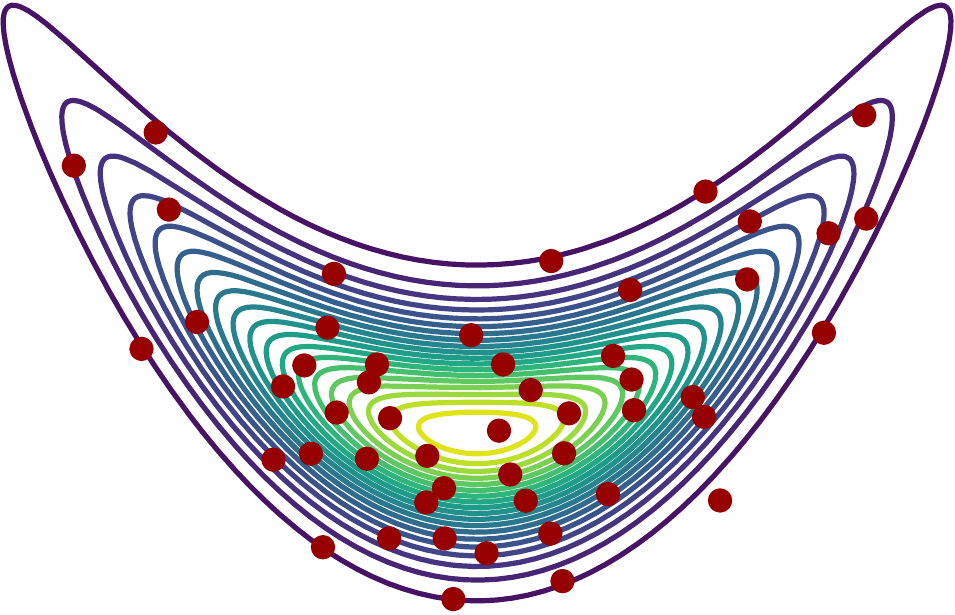}\\
\includegraphics[width=.46\linewidth]{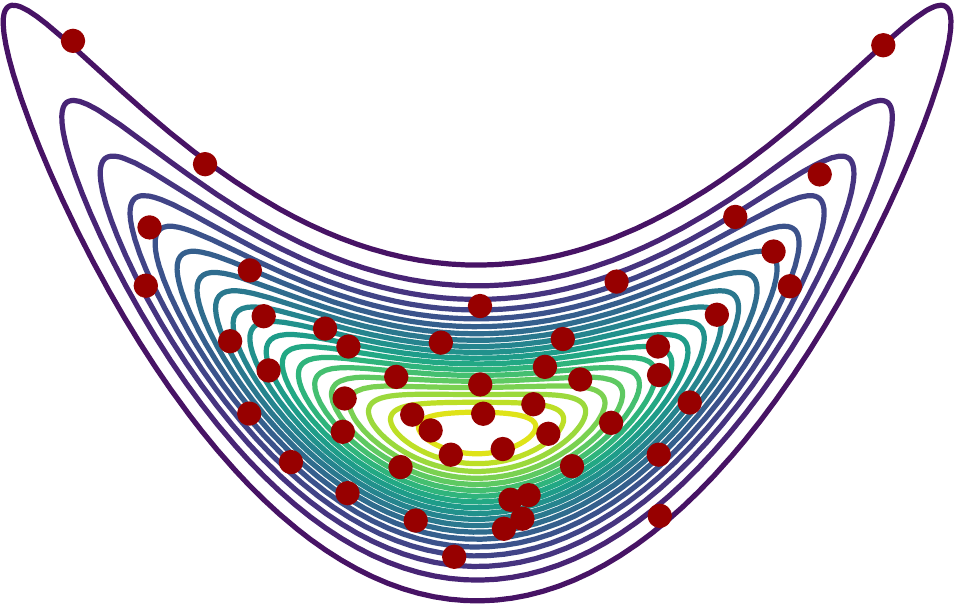}&
\includegraphics[width=.46\linewidth]{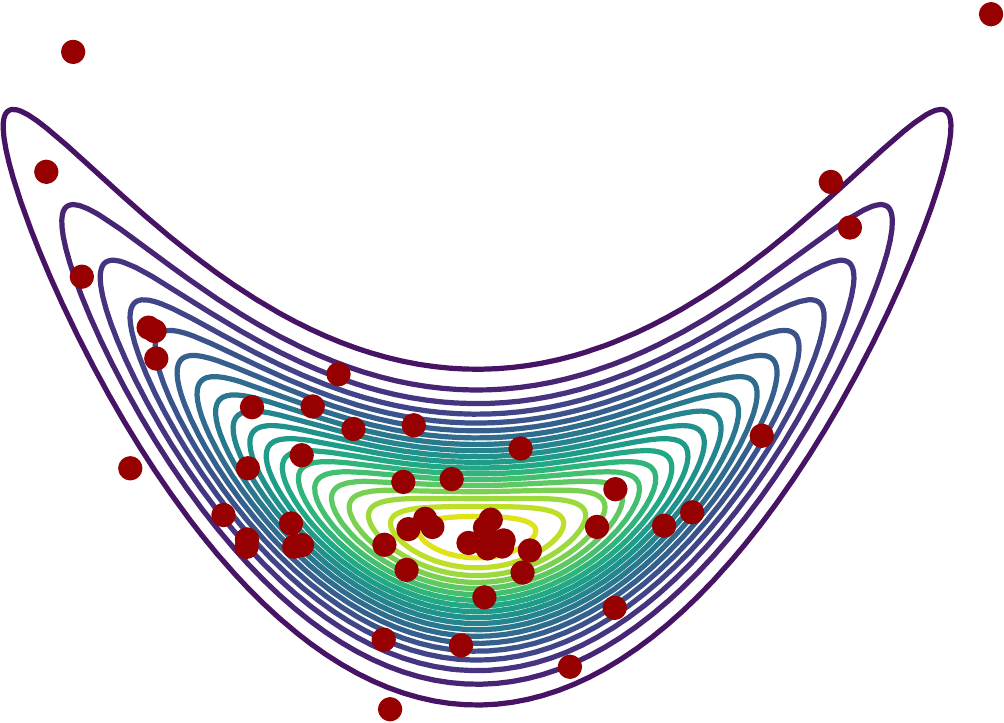}
\end{tabular}
    \caption{Coresets with 50 points for a Gaussian (top left) and the pushforward of a Gaussian through $f:(x,y)\mapsto(x,x^2+y)$. Top right is the image of the Gaussian coreset through $f$, bottom left is computed directly on the pushforward. A random sample is plotted bottom right. \vspace{-.2in}}
    \label{fig:pushcoreset}
\end{figure}


\section{Introduction}
\label{sec:intro}

How do we deal with too much data? Despite the common wisdom that more data is better, algorithms whose complexity scales with the size of the dataset are still routinely used in many areas of machine learning. While large datasets capture high frequency differences between data points, many algorithms only need a handful of \emph{representative} samples that summarize the dataset.

Formalizing a notion of \emph{representative} requires care, however, since a representative sample for a clustering algorithm may differ from that for a classification algorithm. The notion of a data \emph{coreset} was introduced to specify precisely a notion of data summarization that is task dependent. Originally proposed for computational geometry, coresets have found their way into the learning literature for tasks ranging from clustering \citep{DBLP:conf/aistats/BachemLL18}, classification \citep{DBLP:journals/jmlr/TsangKC05}, neural network compression \citep{DBLP:journals/corr/abs-1804-05345}, and Bayesian inference \citep{DBLP:conf/nips/HugginsCB16,Campbell19_JMLR}.

Coreset construction is typically posed as a discrete optimization problem: Given a fixed dataset and learning algorithm, how can we construct a smaller dataset on which that algorithm achieves similar performance? This approach, however, ignores a key theme in machine learning. A dataset is an empirical sample from an underlying data distribution, and learning problems typically seek to minimize an expected loss against the distribution, not the dataset. The effectiveness of a coreset should thus be measured against the \textit{distribution}, and not the \textit{sample}. In other words, the coreset should be designed to guarantee good generalization.

To address this oversight, we introduce \emph{measure coresets}, which approximate the dataset by either a parametric continuous measure or a finitely supported one with a smaller number of points. Our formulation extends coreset language to smooth data distributions and recovers the original formulation when the distribution is supported on finitely many points. We specifically focus on \emph{Wasserstein measure coresets}, which hinge on a natural connection between coreset language and optimal transport theory.

\paragraph*{Contributions.} 
We generalize the definition of a coreset to take into account the underlying data distribution, producing a \emph{measure coreset}, with strong generalization guarantees for a variety of learning problems. Our formulation reveals an elegant connection to optimal transport, allowing us to leverage relevant theoretical results to obtain generalization error bounds for our coresets as well as stability under Lipschitz transformations. 
From a computational perspective, we provide stochastic algorithms for extracting measure coresets, yielding methods that are well-adapted 
%
to cases involving incoming streams of data. 
This allows us to construct coresets in an online manner, without having to store the whole dataset in memory. Besides, contrarily to existing methods which are specific to a given learning problem, our formulation is robust enough so that a given coreset can be used for different tasks.

\subsection{Related work}
We join the probabilistic language of optimal transport with the discrete setting of data compression via coresets.

\textbf{Coresets.} Initially introduced in computational geometry~\citep{agarwal2005geometric}, coresets have found their way to machine learning research via importance sampling~\citep{DBLP:conf/soda/LangbergS10}. Coreset applications are varied, and generic frameworks exist for their construction~\citep{DBLP:conf/stoc/FeldmanL11}. Among the relevant recent applications are $k$-means and $k$-median clustering~\citep{har2004coresets,arthur2007k,feldman2013turning,DBLP:conf/aistats/BachemLL18}, Bayesian inference~\citep{DBLP:journals/corr/abs-1802-01737,DBLP:conf/nips/HugginsCB16}, support vector machine training~\citep{DBLP:journals/jmlr/TsangKC05}, and neural network compression~\citep{DBLP:journals/corr/abs-1804-05345}. 

While coresets are discrete, a sensitivity-based approach to importance sampling coresets was introduced in a continuous setting for approximating expectations 
under absolutely continuous measures w.r.t.\ the Lebesgue measure~\citep{DBLP:conf/soda/LangbergS10}. For more information, see \citep{DBLP:conf/aistats/BachemLL18,munteanu2018coresets}. 

Another line of work closer to ours uses the theory of Reproducing Kernel Hilbert Spaces (RKHS) to design coresets, in particular kernel herding \cite{DBLP:conf/uai/ChenWS10, DBLP:conf/aistats/Lacoste-JulienL15} and Stein points \cite{chen2018stein}. These methods also take into account the underlying distribution of the data, but both require knowledge of that distribution (e.g., the density up to a normalizing constant) while our approach simply assumes sample access.  


\textbf{Optimal transport (OT).} The connection between optimal transport and quantization can be traced back to \citet{pollard1982quantization}, who studied asymptotic properties of $k$-means in the language of OT. More recently, \citet{cuturi_fast_2014} proposed a more efficient version of transport-based quantization using entropy-regularized transport. Entropy-regularized transport~\citep{DBLP:conf/nips/Cuturi13} is a computationally efficient formulation of OT, which led to a wide range of machine learning applications; see recent surveys \citep{solomon2018optimal,peyre2018computational} for details. Recent results characterize its statistical behavior \cite{genevay2018sample} and its ability to handle noisy datasets \cite{rigollet2018entropic}, which we can leverage to design robust coresets.

Our coreset construction algorithms 
are inspired by semi-discrete methods that compute transport from a continuous measure to a discrete one using power diagrams~\citep{aurenhammer1987power}. Efficient algorithms that use computational geometry tools to perform gradient iterations to solve the Kantorovich dual problem have been introduced for 2D~\citep{merigot2011multiscale} and 3D~\citep{levy_numerical_2015}. Closer to our method are the algorithms by \citet{de2012blue} and \citet{DBLP:journals/corr/abs-1802-05757}, which solve a non-convex problem for the support of a discrete uniform measure that minimizes transport cost to an input image~\citep{de2012blue} or the barycenter of the input distributions~\citep{DBLP:journals/corr/abs-1802-05757}. Stochastic approaches for semi-discrete transport, both standard and regularized, were tackled by \citet{genevay_stochastic_2016}.

\paragraph{Notation.} In what follows, we will consider a compact metric space $\Xx \subseteq \mathbb{R}^d$ endowed with the Euclidean norm on $\mathbb{R}^d$ denoted by $\|\cdot\|$. For a random variable $X$ and a probability distribution $\mu$ on $\Xx$, we denote by $X\sim\mu$ the fact that $X$ has distribution $\mu$. The notation $\mathbb{E}_\mu(X)$ is the expectation of the random variable $X$, when $X\sim\mu$. We denote by $f \sharp \mu$ the pushforward of a measure $\mu$ by $f$. We recall that by definition, $\int_\Xx x \,\mathrm{d}(f\sharp \mu) = \int_\Xx f(x)\,\mathrm{d}(\mu)$.

\section{Coresets: From Discrete to Continuous}
\label{sec:coresets}

\subsection{Discrete coresets}
\label{sec:discrete-coresets}
A coreset is a \emph{small summary} of a data set. \emph{Small} usually refers to a the number of points in the coreset, which one hopes is much smaller than the data set size, but one can also think of this in terms of the number of bits required to store the coreset. The \emph{summary} is often a weighted subset of the data, but can also refer to points that are not in the initial dataset but rather represent the original points well.

To make these notions more precise, we must define a coreset in terms of both the dataset and the cost function that the coreset is meant to perform well against. We can understand the definition as a learning problem, where our goal is to approximate the performance of a learning algorithm on a dataset $X$ by its performance on the coreset $C$.

Let $\mathcal{F}$ be the hypothesis set for a learning problem. Every function $f \in \mathcal{F}$ maps from $X$ to $\R$. Let $\mu_X$ be a weighting function on the points in $X$ (this is typically uniform), and define the cost of $f$ on $(X, \mu_X)$ as
\begin{align}
    \mathrm{cost}(X, \mu_X, f) = \sum_{x\in X} \mu_X(x) f(x). \label{eq:cost}
\end{align}
A coreset is then defined by a set $C$ and a weight function $\mu_C$ in such a way that $\mathrm{cost}(C, \mu_C, f)$ is close to $\mathrm{cost}(X, \mu_X, f)$. This leads to the following classical definition of a coreset \cite{bachem2017practical}:

\begin{definition}[Strong/weak $\varepsilon$-coreset]
  \label{def:dcoreset}
The pair $(C, \mu_C)$ is a \emph{strong} $\varepsilon$-coreset for the function family $\mathcal{F}$ if $C \subseteq X$ and
$ 
  \left| \mathrm{cost}(X,\mu_X, f) -  \mathrm{cost}(C,\mu_C, f)\right| \leq \varepsilon\cdot \mathrm{cost}(X,\mu_X, f)
$
for all $f \in \mathcal{F}$. 
If we require that the inequality only holds at $f^* = \argmin_{f\in \mathcal{F}} \mathrm{cost}(X,\mu_X,f)$, then we call $(C, \mu_C)$ a \emph{weak} $\varepsilon$-coreset.
\end{definition}


A coreset always exists for a dataset $(X, \mu_X)$ and 
family $\mathcal{F}$ as the original dataset $(X, \mu_X)$ satisfies Definition \ref{def:dcoreset}.

What distinguishes coresets from other notions of data sparsification is their dependence on the learning problem. For instance, there exist coresets for clustering \cite{DBLP:conf/kdd/BachemL018, DBLP:conf/aistats/BachemLL18}, Bayesian inference \cite{Campbell19_JMLR}, and classification \cite{DBLP:journals/corr/abs-1708-03835}.

\paragraph{Example ($k$-means).} The cost of a particular choice $Q$ of $k$ centers is given by $\sum_{x\in X} \min_{q\in Q} \|x - q\|^2$. To translate this into the language of Definition \ref{def:dcoreset}, we take $f_Q(x) = \min_{q\in Q} \|x - q\|^2$ and $\mu_X(x) = 1$ for all $x\in X$. The function family $\mathcal{F}$ is thus parameterized by the set of all possible choices of the center set $Q$, and we wish to construct a coreset that performs well against all such choices (in the case of a strong coreset) or against the optimal $k$-means assignment (in the case of a weak coreset).

\subsection{Measure coresets}
\label{sec:measure-coresets}

So far we have used discrete language to describe coresets, but this belies the intent of coresets for learning problems. Typical learning problems are posed as minimizations in a hypothesis class of an \emph{expectation} over a data distribution $\mu$. The standard coreset definition is incompatible with this setting as it relies on the existence of a finite data set.
To circumvent this issue, we define a \emph{measure coreset} as a measure $\nu$ that produces similar results under $\mathcal{F}$ as $\mu$: 
\begin{definition}[Measure Coreset]
  \label{def:cmcoreset}
  We call $\nu$ a \emph{strong} $\varepsilon$-measure coreset for $\mu$ if for all $f \in \mathcal{F}$
  \begin{align}
    \label{eq:mcorecost}
    \left|\mathbb{E}_{\mu}[f(X)] - \mathbb{E}_{\nu}[f(X)]\right| \leq \varepsilon.
  \end{align}
\end{definition}
In analogy to the discrete case, a \emph{weak} $\varepsilon$-measure coreset is one for which the inequality holds at $f^* = \argmin_{f\in\mathcal{F}}\mathbb{E}_{\mu}[f(X)]$.
As in the case of discrete coresets, such a $\nu$ always exists, as $\nu = \mu$ satisfies the inequality.

Beyond the change to measure theoretic language, our definition differs from the typical coreset one in two ways. (1) The coreset $\nu$ can be an absolutely continuous measure, which means the size of the coreset can no longer be measured simply in the number of points. (2) We use absolute error instead of relative error; this connects our notion of coreset with generalization error in learning problems in that we can see the coreset as \emph{observed} data and the full measure as \emph{out of sample} data. Absolute instead of relative error is uncommon in coreset language, but not unheard of; see \cite{reddi2015communication, DBLP:conf/kdd/BachemL018} for examples.

Under which constraints on $\nu$, $\mu$ and $\mathcal{F}$ can we construct a measure coreset? We will show a connection to optimal transport and a resulting construction algorithm that aims at minimizing a Wasserstein distance between the coreset $\nu$ and the target measure $\mu$. Using optimal transport duality, we can qualify which learning problems admit measure coresets and the guarantees we can hope to achieve.

\section{Sufficient Conditions for Coreset Approximation}\label{sec:coresetconstruction}
The link between our measure coreset formulation and the theory of optimal transport uses the notion of integral probability metrics \cite{muller1997integral}:
\begin{definition}[Integral Probability Metric] Consider a class of functions $\mathcal{F}:\Xx\to\R$. The integral probability metric $d_\mathcal{F}$ between two measures $\mu$ and $\nu$ is defined by
 \begin{equation}
       d_\mathcal{F}(\mu, \nu) = \sup_{f \in \mathcal{F}} \left|\mathbb{E}_{ \mu}[f(X)] - \mathbb{E}_{ \nu}[f(X)]\right|. \label{eq:IPM}
 \end{equation}
 \end{definition}
Under mild assumptions on the set of functions $\mathcal{F}$, $d_\mathcal{F}$ defines a distance on the space of probability measures. We mention the following examples:
\begin{itemize}
  \item 1-Wasserstein Distance:
  $\mathcal{F} = \{f \ |\ \|\nabla f\| \leq 1\}$ the space of 1-Lipschitz functions.
  \item Dual-Sobolev distance: $\mathcal{F} = \{f \ |\ \|f\|_{H^1(\mu)} \leq 1\}$ where $H^1$ is the Sobolev space $\{f \in L^2 \mid \partial_{x_i} f \in L^2\}$.
  \item Maximum Mean Discrepancy (MMD) \cite{DBLP:journals/corr/abs-0805-2368}:  $\mathcal{F} = \{f \ |\ \|f\|_{\mathcal{H}} \leq 1\}$ where  $\mathcal{H}$ is a universal Reproducing Kernel Hilbert Space (RKHS).
\end{itemize}

The examples above allow us to derive a coreset condition for each of these function classes based on the Wasserstein distance or the MMD, explored in detail below.

\paragraph{Wasserstein distances.}
The $p$-Wasserstein distance between distributions $\mu$ and $\nu$ is given by the solution of a minimization problem:
\begin{equation}
  \label{eq:wp}
  W_p^p(\mu, \nu) = \inf_{\pi\in \Pi(\mu, \nu)} \int_{\Xx \times \Xx}  \|x - y\|^p \,\mathrm{d}\pi(x, y),
\end{equation}
where $\Pi(\mu, \nu) = \{\pi\in P(\Xx\times \Xx)\ |\ \pi(\mathrm{d}x\times \Xx) = \mu(\mathrm{d}x), \pi(\Xx\times \mathrm{d}y) = \nu(\mathrm{d}y)\}$ is the set of couplings with marginals $\mu$ and $\nu$.

When $p = 1$, $W_1(\mu, \nu)$ can be rewritten via duality as a maximization problem over the set of $1$-Lipschitz functions \citep[\S3.1]{santambrogio_optimal_2015}. In particular, for $\mathcal{F} = \mathrm{Lip}_1(\Xx)$, 
\begin{align*}
  d_\mathcal{F}(\mu, \nu) = \sup_{f\in \mathrm{Lip}_1}\int_\Xx f\,\mathrm{d}(\mu-\nu) = W_1(\mu, \nu).
\end{align*}

When $p = 2$,  $W_2(\mu, \nu)$ upper-bounds the dual Sobolev norm of $(\mu - \nu)$ if $\mu$ and $\nu$ have densities w.r.t\ the Lebesgue measure that are bounded above by some constant $M$. In particular, for any $C^1$ function $f$, define a semi-norm by
\begin{align*}
  \|f\|_{H^1(\mu)} = \left(\int_\Xx |\nabla f(x)|^2\,\mathrm{d}\mu(x) \right)^{\frac{1}{2}}.
\end{align*}
This norm allows us to define a dual Sobolev norm on measures as
\begin{align*}
  \|\nu\|_{H^{-1}(\mu)} = \sup_{\|f\|_{H^1(\mu)} \leq 1} \int_{\Xx} f(x)\,\mathrm{d}\nu(x).
\end{align*}
Using \citep[Equation (17)]{peyre2018comparison}, we obtain that for $\mathcal{F} = \{f \ |\ \|f\|_{H^1(\mu)} \leq 1\}$ :
$$d_\mathcal{F}(\mu, \nu) = \|\mu - \nu\|_{H^{-1}(\mu)} \leq \sqrt{M} W_2(\mu, \nu),$$
where $M$ is the uniform bound on the densities of $\mu$ and $\nu$.

\paragraph{Maximum mean discrepancy.}
When $\mathcal{F}$ is the unit ball of a RKHS, equation \eqref{eq:IPM} defines a distance function known as the \emph{maximum mean discrepancy} \citep{DBLP:journals/corr/abs-0805-2368}. If $\kappa(\cdot, \cdot)$ is the reproducing kernel of the RKHS, we can rewrite \eqref{eq:IPM} as an expectation over kernel evaluations
\begin{align}
  \mathrm{MMD}(\mu, \nu) =& \mathbb{E}_{\mu \otimes \mu} [\kappa(X, X')] + \mathbb{E}_{\nu \otimes \nu}[\kappa(Y,Y')] \nonumber \\ & - 2\mathbb{E}_{\mu \otimes \nu}[\kappa(X, Y)].   \label{eq:mmd-rkhs}
\end{align}

While our focus is on coresets under the Wasserstein distance, we mention that coresets that minimize the MMD have been constructed for kernel density estimation \cite{DBLP:conf/compgeom/PhillipsT18}. Generic construction algorithms for sampling to minimize $\MMD$ to a known fixed measure---known as \emph{kernel herding}---have been given by \citet{DBLP:conf/uai/ChenWS10} and \citet{DBLP:conf/aistats/Lacoste-JulienL15}.


\paragraph{Coreset condition.} Using the properties of IPMs above, we summarize conditions for $\nu$ to be an $\varepsilon$-coreset for $\mu$ based on conditions on $\mathcal{F}$.
\begin{proposition}\label{prop:coreset-cond}
  The measure $\nu$ is an $\varepsilon$-coreset for $\mu$ with function family $\mathcal{F}$ if: \begin{enumerate}[label=(\roman*)]
      \item $W_1(\mu, \nu) \leq \varepsilon$ for $\mathcal{F} \subseteq \mathrm{Lip}_1$;
      \item $W_2(\mu, \nu) \leq \nicefrac{\varepsilon}{\sqrt{M}}$ for $\mathcal{F} \subseteq H^1(\mu)$, when $\mu$ and $\nu$ have densities with respect to the Lebesgue measure that are bounded above by $M$; or
      \item $\MMD(\mu, \nu) \leq \varepsilon$ for $\mathcal{F} \subseteq \mathcal{H}$.
  \end{enumerate}
\end{proposition}

We can extend the first two conditions to $\mathrm{Lip}_K$ and $\|f\|_{H^1(\mu)} \leq K$ by scaling $f$ by the Lipschitz or Sobolev constant by a multiplicative $K$ factor. In the remainder of this paper, we will focus on coresets based on Wasserstein distances and will call them \emph{measure coresets} for simplicity. When more precision is required, we will denote by $W_1$ (resp.\ $W_2$, $\MMD$) measure coreset a coreset with function family $\mathrm{Lip}_1$ (resp.\ $H^1(\mu)$, $\mathcal{H}$).

\section{Practical Wasserstein Coreset Constructions}

While \S\ref{sec:coresetconstruction} gives a metric for measuring how close a distribution $\nu$ is to satisfying the coreset condition for a distribution $\mu$, the question of how to compute such a $\nu$ remains.

In our definition, $\nu$ was unconstrained, but for it to be a useful 
coreset for a measure, we should be able to describe it using fewer bits than needed to describe the full measure $\mu$. From a practical point of view, we should also be able to compute expectations under the coreset $\nu$ and at least approximate expectations under $\mu$.

We make a few simplifications. We assume that we can sample from $\mu$ efficiently and that $\mu$ is supported on a compact set $\mathcal{X} \subset \R^d$. This is true of any finite dataset. The simplest notion of a measure coreset is a uniform distribution over a finite point set $x_1, \ldots, x_n$. This leads to the following optimization problem, which will be our focus in this section:
  \begin{equation}
          \min_{(x_1,\dots,x_n)} W_p\left(\frac{1}{n}\sum_{i=1}^n \delta_{x_i}, \mu\right) \tag{$\mathcal{P}$}. \label{eq:csproblem}
  \end{equation}

It is also possible to formulate the problem using a continuous parametric density as a coreset. Given a family of parametric densities $(p_\theta)_{\theta \in \Theta}$ (e.g., Gaussian), we want to find  the parametric distribution $p_{\*\theta}$ that best approximates a measure $\mu$. This can be written simply as 
  \begin{equation}
          \min_{\theta \in \Theta} W_p\left(p_\theta, \mu\right).
  \end{equation}
We experimented with this option using Gaussian mixtures, but the minimization is highly non-convex, and gradient descent algorithms do not converge except in restricted settings (e.g., mixtures with equal weights). We find the simpler problem \eqref{eq:csproblem} sufficient for the applications we consider and leave computation of more general coresets to future research.

\subsection{Properties of empirical coresets}
We address the problem of estimating $n$ the number of points in a coreset $n$ given $\varepsilon$ for $\mu$ an arbitratry measure continuous. Namely, we ask how many samples $n$ we need such that $W_p\left(\mu, \nu\right) \leq \varepsilon$ when $\nu = \sum_{i=1}^n \delta_{x_i}$. 

\textbf{Statistical bounds.}
There exist several theorems for finite sample rates of $W_p$, 
which each focus on specific hypotheses to marginally improve rates. We give a general statement: 
\begin{theorem}[Metric convergence, \citealt{kloeckner_approximation_2012,brancolini2009long,weed2017sharp}]
  \label{thm:convergence}
  Suppose $\mu$ is a compactly supported measure in $\mathbb{R}^d$ and $\nu_n$ is a uniform measure supported on $n$ points drawn from $\mu$. Then $W_p(\nu_n, \mu) \sim \Theta(n^{-\nicefrac{1}{d}})$. Moreover, if $\mu$ has Hausdorff dimension $s < d$, then $W_p(\nu_n, \mu) \sim \Theta(n^{-\nicefrac{1}{s}})$.
\end{theorem}
Thus, both $W_1$ and $W_2$ have finite sample rate $O(n^{-\nicefrac{1}{d}})$. If we assume that $\mu$ is supported on a lower dimensional manifold 
of dimension $s$, we get the improved rate $O(n^{-\nicefrac{1}{s}})$. 

\begin{corollary} If $\nu = \sum_{i=1}^n \delta_{\*x_i}$ with $n = \Theta(\varepsilon^{-s})$ is a globally optimal solution for~\eqref{eq:csproblem}, then $\nu$ is a $\varepsilon$-measure coreset.
\end{corollary}

 While we cannot guarantee this bound in practice since global optimality is NP-hard \citep{DBLP:journals/corr/abs-1802-05757}, empirically we observe that it holds and in fact is an overestimate of coreset size. Note that the theoretically required coreset size is independent of additional variables in the underlying problem, e.g., the number of means in $k$-means.

This bound improves over the best known deterministic coreset size for $k$-means and $k$-median of $O(k\varepsilon^{-d}\log n)$ \citep{har2004coresets}, but we must be careful as our coreset bounds are given in absolute error. For $k$-means and $k$-medians, we are typically in the regime where the full data set has large cost \eqref{eq:cost}, but if that does not hold, the coresets are no longer comparable. 

Better randomized construction algorithms exist for both $k$-means/$k$-median and SVM with sizes that do not have such a strong dependence on dimension. Empirically, our coresets are competitive, and often better than specialized construction algorithms, especially in the small data regime (see Figures \ref{fig:svm}, \ref{fig:kmeans} and \ref{fig:bayesian}).

One useful property of $W_p$ coresets is that given an $\varepsilon-$coreset for a reference measure $\mu$, we immediately have a $L\varepsilon-$coreset for the pushforward measure $f\sharp \mu$, where $L$ is the Lipschitz constant of $f$.

\newpage
\begin{proposition}(Coreset of pushforward measure) Consider a $L$-Lipschitz function $f$. If $\{\*x_i\}_{i=1}^n$ is a $\varepsilon$-measure coreset under $W_p$ for $\mu$, then $\{f(\*x_i)\}_{i=1}^n$ is a $L \varepsilon$-measure coreset under $W_p$ for $f \sharp \mu$.
\end{proposition}

\begin{proof}
$f$ being $L$-Lipschitz implies that
$\|f(x) - f(y) \|^p \leq L^p \|x - y \|^p \quad \forall (x,y) \in \Xx.$ 
Thus, for all 
$\pi \in \Pi(\frac{1}{n}\delta \*x_i, \mu)$,
\begin{align*}
 \int_\Xx \sum_{i=1}^n & \|f(\*x_i) - f(x) \|^p\, d\pi(x_i,x) \\
 &\leq L^p \int_\Xx \sum_{i=1}^n \|\*x_i - x \|^p\,d\pi(x_i,x).
 \end{align*}
Minimizing over $\pi$ on the right hand side and using the definition of a pushforward measure on the left gives
$$
     W_p^p \left(\frac{1}{n} \sum_{i=1}^n \delta_{f(\*x_i)}, f \sharp \mu \right)  \leq L^p W_p^p \left(\frac{1}{n} \sum_{i=1}^n \delta_{\*x_i}, \mu \right).$$
Since $\*x_i$ is a $W_p$ $\varepsilon$-measure coreset for $\mu$, we have $W_p \left(\frac{1}{n} \sum_{i=1}^n \delta_{\*x_i}, \mu \right) \leq\varepsilon$, yielding the desired bound.
\end{proof}

Pushforward measures are ubiquitous in (deep) generative models, which have gained popularity for image generation through GANs \cite{GAN} and VAEs \cite{VAE}. Specifically, new data is generated by pushing uniform or Gaussian noise through a neural network $f$ \cite{genevay2018learning}. The above proposition suggests that if the pushforward function $f$ has bounded variation, constructing a coreset for the source noise and pushing it through $f$ is sufficient to find a `good enough' coreset for the generative model without additional computations. 
This robustness property is illustrated by Figure~\ref{fig:pushcoreset}, where the banana-shaped distribution is the pushforward of a normalized Gaussian $\mathcal{N}$ through $f:(x,y)\mapsto(x,x^2+y)$. Even though the coreset obtained as the image of the coreset of the Gaussian through $f$ performs slightly worse than the coreset computed directly on $f\sharp\mathcal{N}$, it represents the distribution in a more faithful way than a random sample.

We also have the following relationship between being a $W_2$ coreset and being a $W_1$ coreset:
 \begin{remark}Let $\{\*x_i\}_{i=1}^n$ minimize $W_2\left(\frac{1}{n}\sum_{i=1}^n \delta_{\*x_i}, \mu\right)$. Using the inequality between $W_p$ metrics, 
$$ W_1\left( \frac{1}{n}\sum_{i=1}^n \delta_{\*x_i}, \mu \right)  \leq W_2\left( \frac{1}{n}\sum_{i=1}^n \delta_{\*x_i} \mu \right).$$
Thus, if we choose $n$ large enough such that $W_2\left(\frac{1}{n}\sum_{i=1}^n \delta_{\*x_i}, \mu\right) \leq \varepsilon$, then  $\frac{1}{n}\sum_{i=1}^n \delta_{\*x_i}$ is also a $W_1$ $\varepsilon$-measure coreset for $\mu$.
\end{remark}

\subsection{Entropy-regularized Wasserstein distances} \label{sec:reg-was}
The entropy-regularized Wasserstein distance is a popular approximation of the Wasserstein distance, as it is computable with faster algorithms \citep{cuturi2013sinkhorn}. The entropically regularized $p$-Wasserstein distance is
\begin{align}
\begin{split}
  W_{p,\eta}^p(\mu, \nu)\!\!=\!\!\argmin_{\pi \in \Pi(\mu, \nu)}\! \int_{\!\Xx \times \Xx} \hspace{-.26in}\|x-y\|^p\mathrm{d}\pi(x, y) 
  \!+\!\eta  \mathrm{KL}(\pi\| \mu\!\otimes\!\nu).
\end{split}
\end{align}

 As the $\mathrm{KL}$ term is nonnegative, $W_{p,\eta}^p$ upper-bopunds $W_p^p$ for all $p$, and thus any coreset under $W_{1,\eta}$ and $W_{2,\eta}$ is also a coreset under $W_1$ and $W_2$. Due to the entropic term, however, we have $W_{p,\eta}(\mu,\mu) = O(\eta)$ \cite{genevay2018learning}, so even with a large number of samples $n$ in the coreset, it is not always possible to get an $\epsilon$-coreset for $W_{p}$ using $W_{p,\eta}$. In practice, we observe that this regularizer yields mode collapse of the coreset, with the number of modes decreasing as $\eta$ increases.
 
 To alleviate this issue, \citet{genevay2018learning} introduce Sinkhorn divergences, defined via
$$ SD_{p,\eta}(\mu,\nu) = W_{p,\eta}(\mu,\nu) - \frac{1}{2} \left( W_{p,\eta}(\mu,\mu) + W_{p,\eta}(\nu,\nu) \right).$$
The additional terms ensure that $SD_{p,\eta}(\mu,\mu) = 0$. Interestingly, when $\eta$ goes to infinity, Sinkhorn divergences converge to MMD defined in \eqref{eq:mmd-rkhs} with kernel $\kappa(x,y) = - \|x-y\|^p$ for $0<p<2$. While solving \eqref{eq:csproblem} using $SD_{p,\eta}$ can be faster than with $W_{p}$, especially for larger coreset sizes, we do not have theoretical guarantees for the minimizer.

\subsection{Algorithms}
Recall that the goal of our measure coreset algorithms is to find a set of points $\{x_1, \ldots, x_n\}$ that minimizes some Wasserstein distance to a given distribution. Here, we detail how this goal is achieved by leveraging the dual of the Wasserstein problem. In particular, we give algorithms that compute coresets under the $W_1$ and $W_2$, via the updates specific to each setting.

\paragraph{Minimizing $W_1$ and $W_2$.} 

\begin{algorithm}[!t]
  \caption{Compute an online $W_1$ coreset via SGD}
  \label{alg:w1_coreset}
  \begin{algorithmic}[1]
    \REQUIRE{Measure $\mu$, $n > 0$, minibatch size $m$, $\gamma > 0$}
    \ENSURE{Points $x_1, \ldots, x_n$}
    \STATE{Initialize $(x_1,\dots,x_n)\sim \mu$}
    \FOR{$k=1, \ldots$}
    \STATE{Sample $(y_1,\dots,y_m)\sim \mu$}
    \STATE{Update estimate of $v^*$ using samples $y_k$.}
    \STATE{Define generalized Voronoi regions $V_i(v^*)$.}
    \STATE{Step: $x_i \gets x_i - \frac{\gamma}{\sqrt{k}} \sum_{y_k \in V_i(v^*)} \frac{1}{|V_i(v^*)|}\frac{y_k - x_i}{\|y_k - x_i\|}$}. \label{eq:gradientstep}
    \ENDFOR
  \end{algorithmic}
\end{algorithm}

\begin{algorithm}[!t]
  \caption{Compute an online $W_2$ coreset via SGD}
  \label{alg:w2_coreset}
  \begin{algorithmic}[1]
    \REQUIRE{Measure $\mu$, $n > 0$, minibatch size $m$, $\gamma > 0$}
    \ENSURE{Points $x_1, \ldots, x_n$}
    \STATE{Initialize $(x_1,\dots,x_n)\sim \mu$}
    \FOR{$k=1, \ldots$}
    \STATE{Sample $(y_1,\dots,y_m)\sim \mu$}
    \STATE{Update estimate of $v^*$ using samples $y_k$.}
    \STATE{Define generalized Voronoi regions $V_i(v^*)$.}
    \STATE{Update: $x_i \gets \sum_{y_k \in V_i(v^*)} \frac{1}{|V_i(v^*)|}y_k$}.
    \ENDFOR
  \end{algorithmic}
\end{algorithm}

In the semi-discrete case, when $\nu = \frac{1}{n}\sum_{i=1}^n \delta_{x_i}$,  computing the Wasserstein distance can be cast as maximizing an expectation: 
\begin{align}
  W_p^p \left(\nu , \mu \right) = \max_{v\in \R^n} \mathbb{E}_{\mu} \Big[ \min_{i} & \left(\|X - x_i\|^p  - v_i\right)+ \frac{1}{n} \sum_{i=1}^n  v_i \Big],
  \label{eq:semidiscrete}
\end{align}
which can be optimized via stochastic gradient methods \citep{genevay_stochastic_2016,DBLP:journals/corr/abs-1802-05757}. 
The gradients w.r.t.\ $x_i$ can be written in terms of power 
diagrams:
\begin{align}
    \nabla_{x_i} W_1\left(\frac{1}{n} \sum_{i=1}^n \delta_{x_i}, \mu \right) &= \int_{V_i(v^*)} \frac{x - x_i}{\|x - x_i\|}\,\mathrm{d}\mu(x) \label{eq:gradw1}\\ 
    \nabla_{x_i} W_2^2\left(\frac{1}{n} \sum_{i=1}^n \delta_{x_i}, \mu \right) &= x_i - \int_{V_i(v^*)} x\,\mathrm{d}\mu(x) \label{eq:gradw2}
\end{align}
where $v^*$ is the solution of \eqref{eq:semidiscrete} and $V_i(v) = \{x : \|x - x_i\|^p - v_i \leq \|x - x_j\|^p - v_j, \forall j \neq i\}$ is the generalized Voronoi region of point $x_i$ with $p = 1$ for $W_1$, and $p = 2$ for $W_2$.

Thus, a gradient step in the point positions $x_i$ requires first solving \eqref{eq:semidiscrete} to get the optimal $v$, and then computing the gradients according to \eqref{eq:gradw1}, \eqref{eq:gradw2}. For $W_2^2$, the gradient step can be replaced by a fixed point iteration  \cite{DBLP:journals/corr/abs-1802-05757}.

\paragraph{Minimizing $W_{p,\eta}$ and $SD_{p,\eta}$.} Due to the mode collapse inherent to large regularization $\eta$ mentioned in \S\ref{sec:reg-was}, Sinkhorn divergences empirically are better candidates to construct coresets. Following \cite{genevay2018learning}, we compute $\nabla_x SD_{p,\eta}$ using automatic differentiation of the objective. The resulting algorithm is identical to Algorithm~\ref{alg:w1_coreset}, where $\nabla_x W_1$ gradient in line~(6) is replaced by $\nabla_x SD_{p,\eta}$.




\subsection{Convergence}
We mention some observations on the convergence of our approach. The minimization over the $x$ variables is not convex due to inherent symmetries in the solution space, and $W_p(\cdot, \cdot)$ is not sufficiently smooth in the $x$ variables to give precise convergence guarantees. 

In Algorithms \ref{alg:w1_coreset} and \ref{alg:w2_coreset}, we specify the number of points in the coreset. This parameter is unlike discrete coreset algorithms, which take $\varepsilon$ as an input and return a coreset with enough points to satisfy the coreset inequality. Because our input is a measure that is absolutely continuous with respect to the Lebesgue measure, we do not have the luxury of this approach. An illustrative example is to consider $\varepsilon = 0$. In this case, a discrete coreset algorithm would simply return the original dataset. For a continuous $\mu$, however, there is no finite distribution that has $0$ error relative to $\mu$.




\subsection{Implementation details}
\label{sec:analysis:cons}
 Construction time depends strongly on the characteristics of the measure we are approximating. Most of the time is spent evaluating the expectations in \eqref{eq:gradw1}, \eqref{eq:gradw2}. Since we run the gradient ascent until $\|\nabla_w F\|_2 \leq \varepsilon$ and perform $T$ fixed point iterations,
the construction requires $O(\nicefrac{T}{\varepsilon})$ calls to an oracle that computes densities of the power cells
$V_i(v)$. 

The algorithms for $W_1$ and $W_2$ were implemented in C++ using the Eigen matrix library \cite{eigenweb} and run on an Intel i7-6700K processor with 4 cores and 32GB of system memory. Computing expectations under samples from $\mu$ can be trivially parallelized. The total coreset construction time ranges from a few seconds for small coresets on small datasets, to 5 minutes on large datasets where large coresets are required. The Sinkhorn divergence coresets were implemented in TensorFlow and run on the same architecture without GPU support. Since our code for $W_p$ is in C++, we do not observe significant computational speedup when using Sinkhorn divergences in our experiments. As the resulting coresets are merely an approximation of $W_p$ coresets, we do not display them in the experimental results.

All algorithms were run 20 times -- we display the mean and standard deviations in our plots. Regarding the parameters in Algorithms \ref{alg:w1_coreset} and \ref{alg:w2_coreset}, we use a step size $\gamma=1$ and 100 iterations.


%
\begin{figure}[!tb]
\centering
\includegraphics[width=0.95\columnwidth]{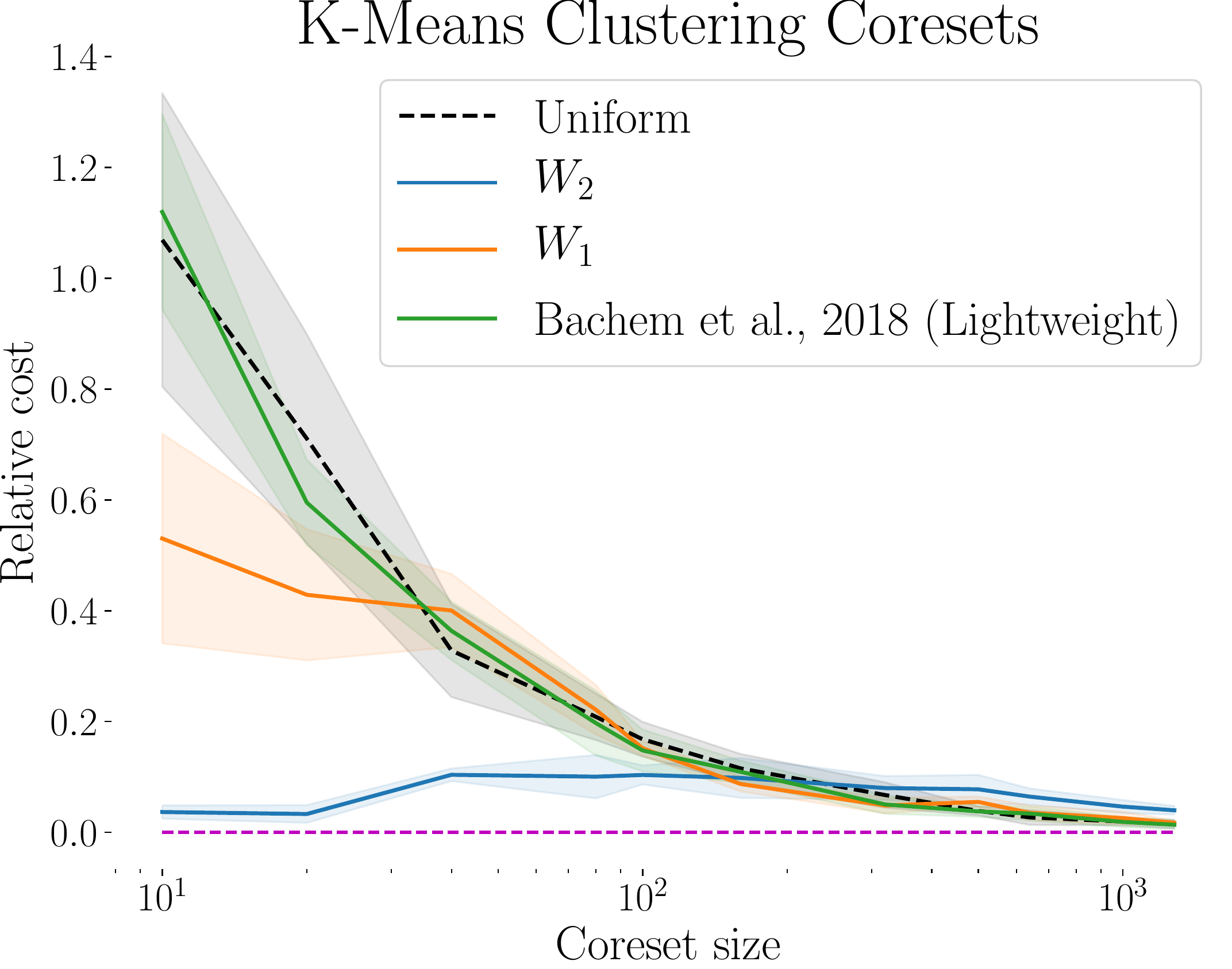}
\caption{Coreset construction on the Pendigit dataset \cite{keller2012hics} for the $k$-means algorithm. We compute the $k$-means cost on the full data using means learned on the coreset. The $y$ axis measures relative error to computing the cost using the means learned on the full data. Comparison is with \cite{DBLP:conf/kdd/BachemL018}. We expect (and verify) that $W_2$ coresets perform better than $W_1$ coresets on this problem.}
\label{fig:kmeans}
\end{figure}
\begin{figure}[!tb]
\centering
\includegraphics[width=0.95\columnwidth]{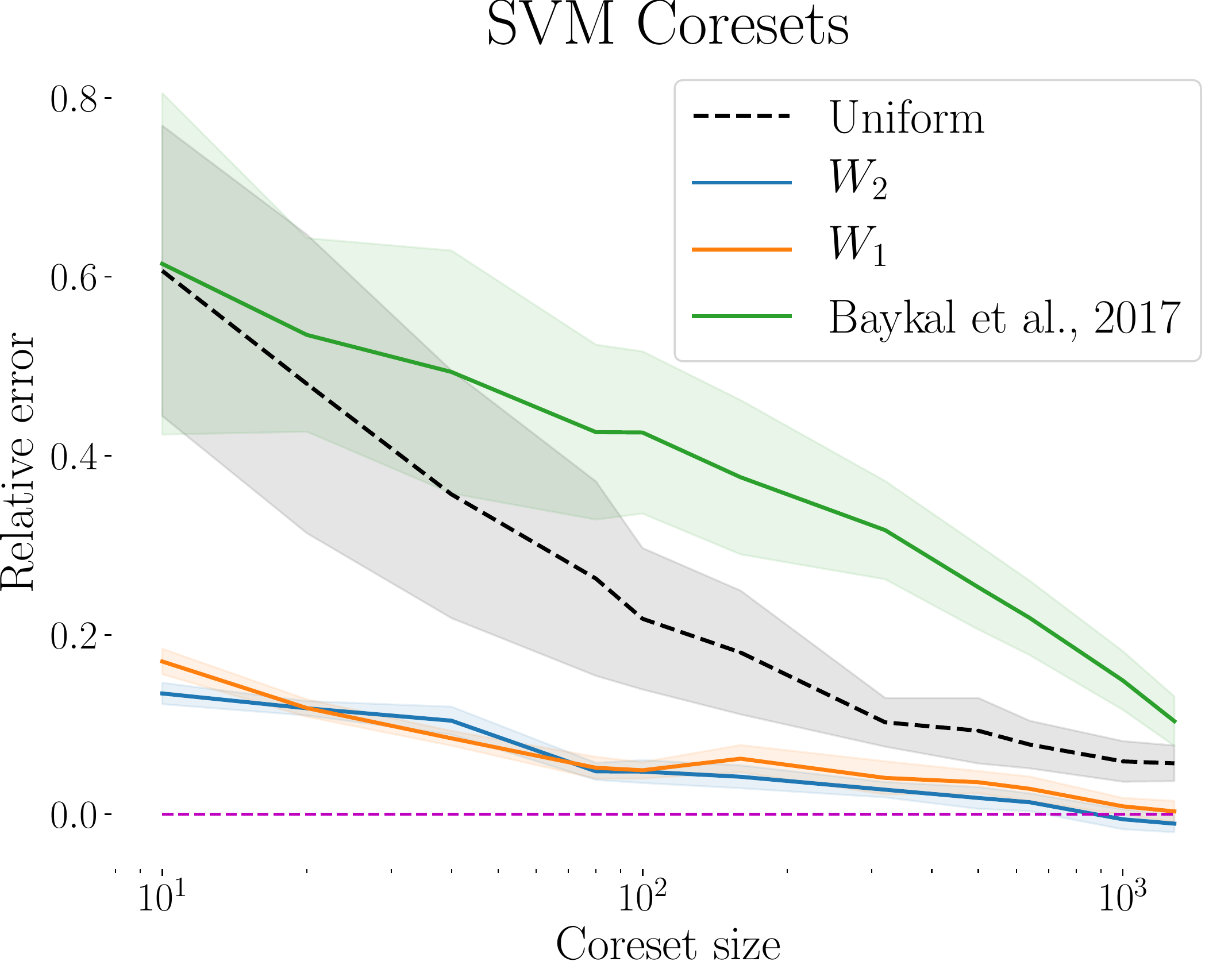}
\caption{Coreset construction on the UCI credit card dataset \cite{DBLP:journals/eswa/YehL09a} for SVM classification. We compute relative accuracy with respect to training a classifier on all the data. Comparison is with \cite{DBLP:journals/corr/abs-1708-03835}. Soft margin SVMs minimize a Lipschitz cost, and we expect both $W_1$ and $W_2$ coresets to perform well.}
\label{fig:svm}
\end{figure}
\begin{figure}[!tb]
\centering
\includegraphics[width=0.95\columnwidth]{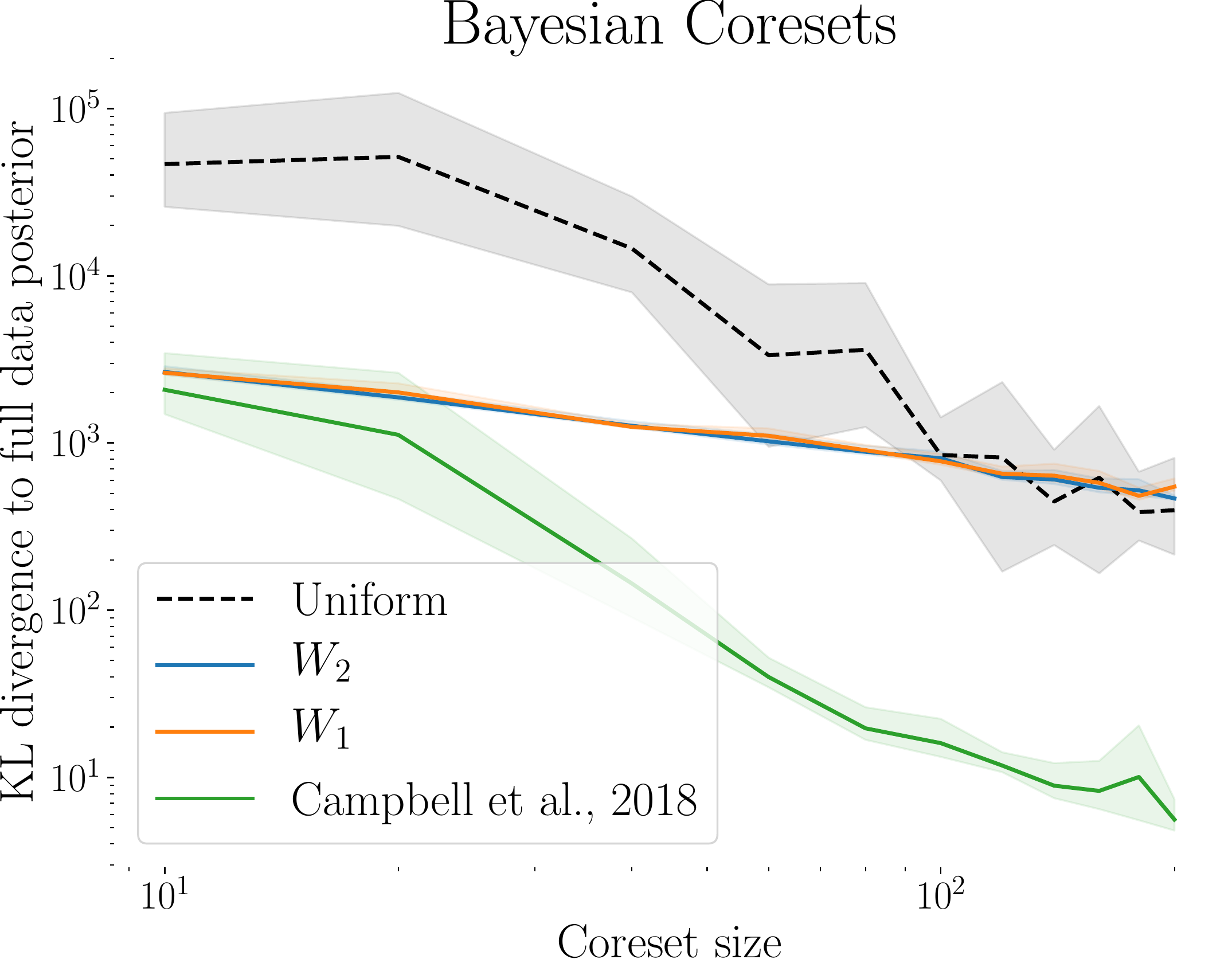}
\caption{Coreset construction on a synthetic dataset (described in \ref{sec:bayesian-core}). The goal is to approximate the posterior distribution for a logistic regression model, and we report the $\mathrm{KL}$ divergence to the true posterior learned on the full data. Comparison is with \cite{Campbell19_JMLR}. The $\log$ likelihood of the model is Lipschitz, and we expect similar performance from $W_1$ and $W_2$ coresets.}
\label{fig:bayesian}
\vspace{-.2in}
\end{figure}

\section{Comparison with Classical Coresets}
\label{sec:experiments}
We compare with classical coreset constructions on a few problems. Each of the three tasks we consider has a specialized coreset construction algorithm that does not extend to other problems.  
Our coresets, on the other hand, do not have this limitation, but broad applicability may come at the price of performance. Even so, our coresets perform better than uniform on the three tasks we have chosen ($k$-means clustering, SVM classification, posterior inference), and greatly \textit{outperform} state-of-the-art algorithms for the first two.

\subsection{$k$-means clustering}
\label{sec:kmeans-core}
The $k$-means objective for a fixed set of cluster centers $Q$ is given by $J(Q)=\sum_{x\in X}\min_{q\in Q} \|x - q\|^2$.

When $Q$ is a subset of a compact set,  this cost has bounded Sobolev norm but is not Lipschitz. We expect $W_1$ coresets to perform worse than $W_2$ coresets on this problem. To measure performance, we compute coresets on the Pendigits dataset \cite{keller2012hics} and compute relative cost $1-J(Q_{c})/J(Q^*)$ of the centers learned on the coreset $Q_{c}$ against the centers learned on the full data $Q^*$. We compare with the importance sampling method of \citet{DBLP:conf/kdd/BachemL018}. The number of clusters we expect in the data is $10$, one for each digit.

In this experiment, \cite{DBLP:conf/kdd/BachemL018} does not exhibit a clear advantage over uniform sampling. This suggests that their method is better suited to larger datasets. On the other hand, when using $W_2$ coresets, our method is on par with the minimal error for a coreset of $10$ points. This is not surprising, as minimizing \eqref{eq:csproblem} with $W_2$ and $n=k$ support points is equivalent to minimizing the $k$-means objective with balanced cluster assignments \cite{pollard1982quantization, canas_learning_2012}. This example demonstrates that our stochastic gradient descent approach is an efficient means of solving balanced $k$-means problems over large datasets, since we only access small-sized batches of the data at each iteration and never process the whole dataset at once.

\subsection{Support vector machine classification}
\label{sec:svm-core}
The soft margin SVM cost of a point $x_i$ with label $y_i$ is given by $y_i (w^\intercal x_i + b) - 1 + \xi_i,$
where $\xi_i$ is a slack variable associated to $x_i$. This cost is Lipschitz with a constant depending on the diameter of the set of allowable $w$'s. 

Because SVMs solve classification problems and our coresets approximate a dataset, our experimental setup here is slightly different than for $k$-means. Instead of constructing a coreset on the $(x_i, y_i)$ pairs in the training data, we construct individual coresets for all data associated to a single label and merge them afterward. 
Hence, the coreset contains equal numbers of positive and negative samples. We hypothesize that this property and the tendency of coresets to remove large outliers explains why in Figure \ref{fig:svm} our coresets can yield better classifiers than training on the full data for large coreset size.

\subsection{Bayesian inference}
\label{sec:bayesian-core}
We construct a synthetic dataset for logistic regression by drawing $x_i \sim \mathcal{N}(0, I)$ and labeling the $x_i$ by
\begin{align}
    \theta\sim \mathcal{N}\left(0, I\right)\quad y_i\: |\: x_i, \theta \sim \mathrm{Bern}\left(\frac{1}{1 + e^{-x_i^\intercal \theta}}\right).
\end{align}
The goal is to construct a (weighted) coreset that approximates the $\log$ likelihood of the full data $\sum_i \log p(y_i\: |\: \theta)$. This cost is Lipschitz in this particular case. To agree with \cite{Campbell19_JMLR}, instead of computing the relative $\log$ likelihood of our coreset against that of the full data, we use the coreset to infer the parameters of the posterior distribution and report $\mathrm{KL}$ divergence against the posterior learned on the entire dataset. 
Figure \ref{fig:bayesian} shows results on a dataset of $20000$ points drawn from a 5-dimensional Gaussian distribution. While we do not match the performance of \cite{Campbell19_JMLR}, our coreset performs significantly better than a uniform sample.
\section{Discussion}
\label{sec:conclusion}

Learning problems are frequently posed as finding the best hypothesis that minimizes expected loss under a data distribution. 
However classic coreset theory ignores that the samples from the dataset are drawn from some distribution. We have introduced a notion of \emph{measure coreset} whose goal is to minimize generalization error of the coreset against the data distribution. Our definition is the natural one, and we can draw connections between this generalized notion of a coreset and optimal transport theory that leads to online construction algorithms. 

As our paper is exploratory, there are many avenues for future research. For one, our definitions rely on identities and inequalities that relate large function families to  $W_1$ and $W_2$. If we cannot assume much about $\mu$, then these relations cannot be refined. The theory in our paper, however, does not sufficiently explain the effectiveness of our coreset constructions on the learning problems in \S\ref{sec:experiments}. 

Our algorithm's performance suggests several questions. There is a gap between the statistical knowledge we have about the sample complexity of $W_1$ and $W_2$ and the behavior of Algorithms \ref{alg:w1_coreset} and \ref{alg:w2_coreset} in the few-samples regime. Additionally, being able to get a coreset condition similar to Proposition~\ref{prop:coreset-cond} for Sinkhorn divergences would allow us to leverage their improved sample complexity compared to Wasserstein distances, yielding tighter theoretical bounds for the number of points required to be an $\epsilon$-measure coreset.




\bibliography{biblio_coresets}
\bibliographystyle{icml2020}

\appendix
\clearpage
\begin{figure}[!t]
\includegraphics[width=.49\textwidth]{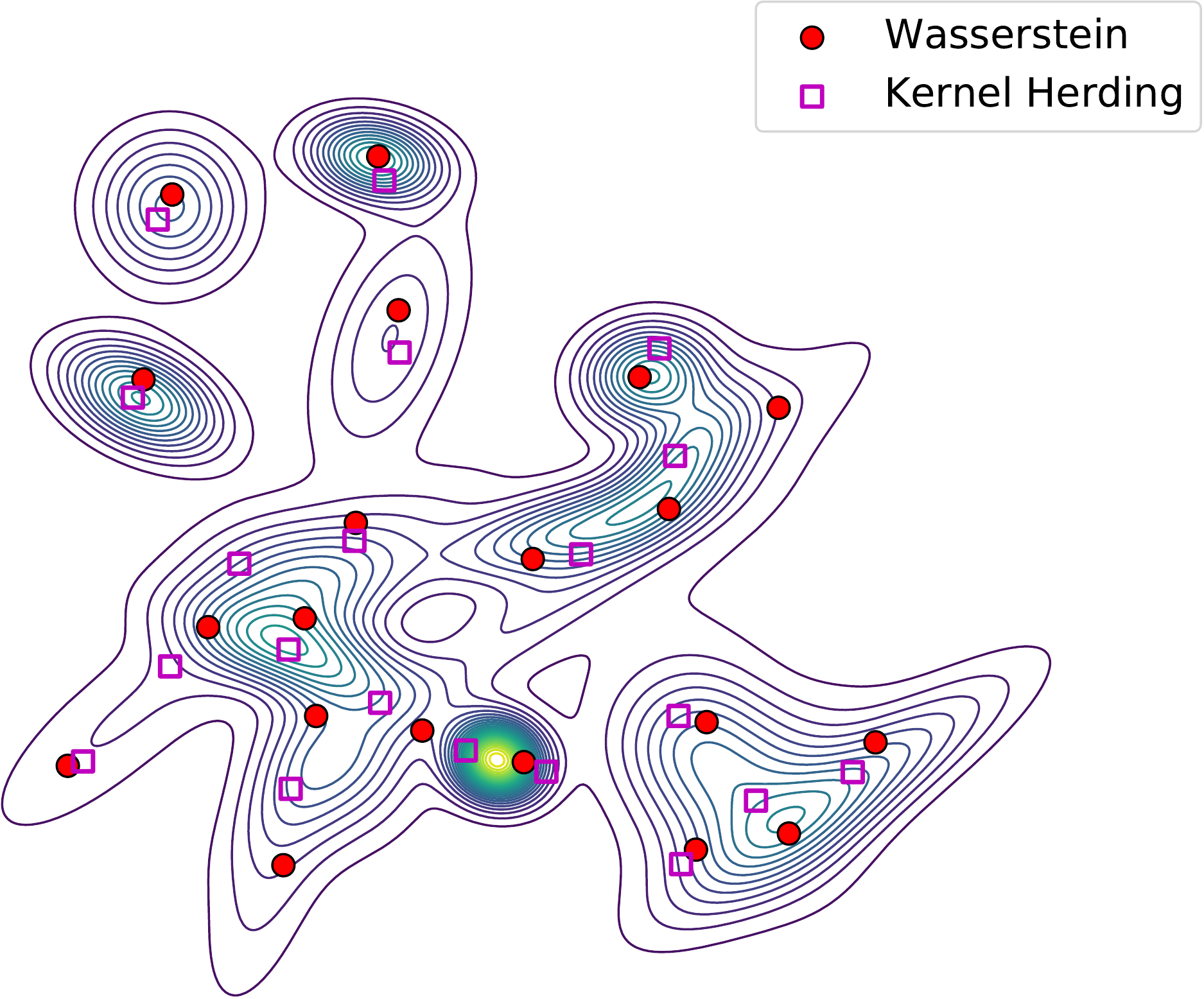}
\caption{Comparison with kernel herding on a mixture of Gaussians. The first twenty points obtained from herding are plotted against a twenty point coreset under the $W_2$ distance.}
\label{fig:herding}
\end{figure}

\begin{figure*}[tb]
\includegraphics[width=.49\textwidth]{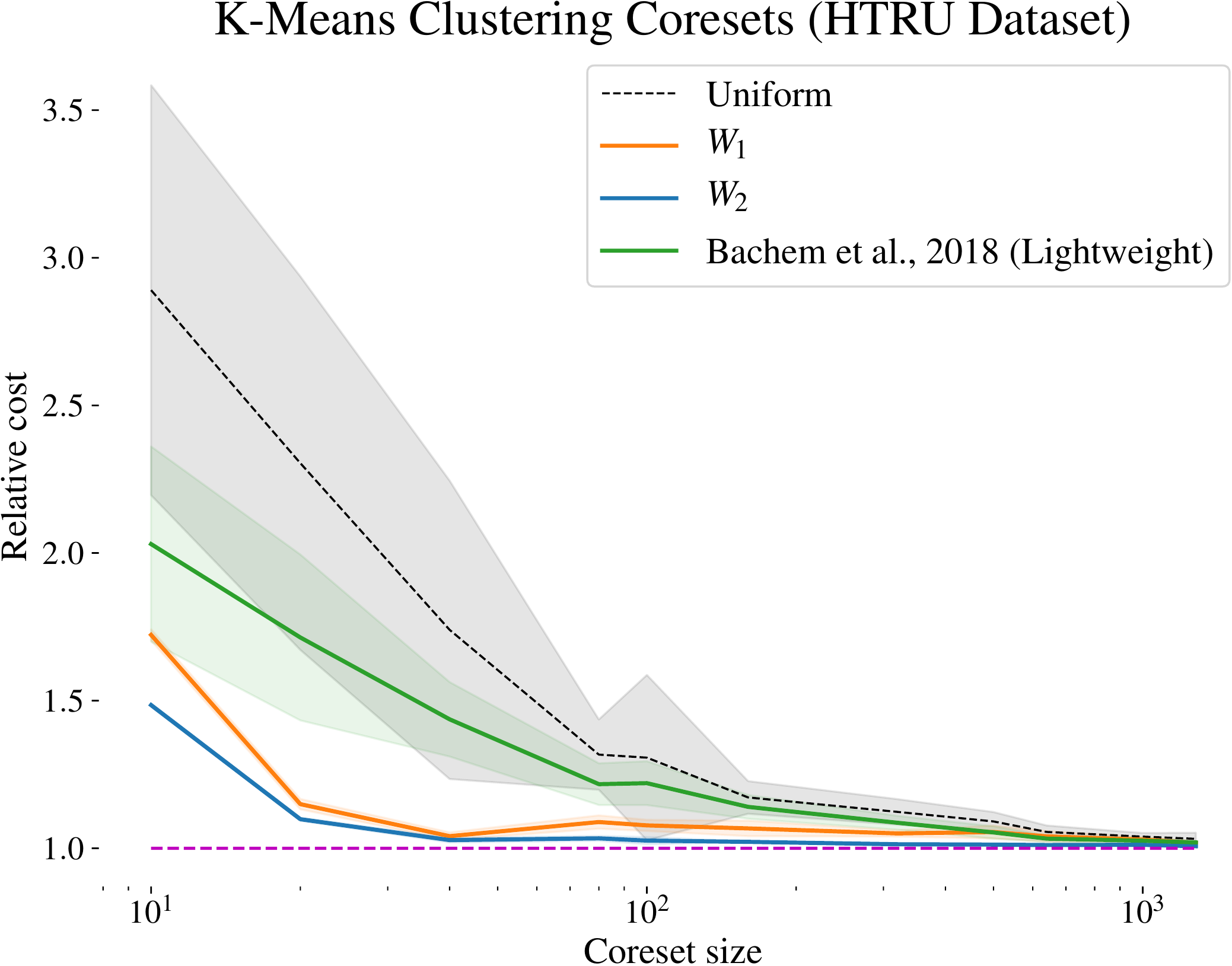}
\includegraphics[width=.49\textwidth]{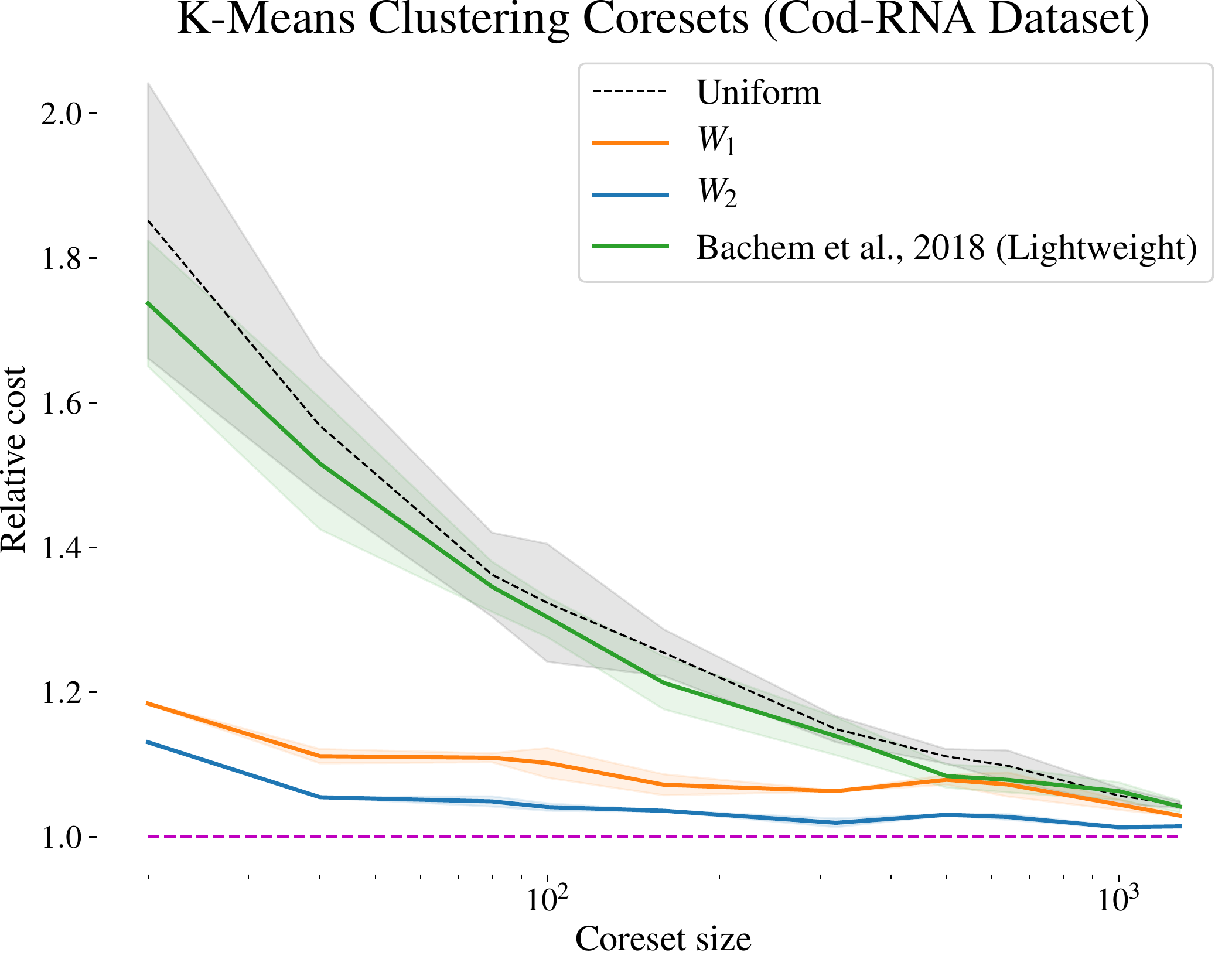}
\caption{Additional results for $k$-means clustering.}
\label{fig:kmeans-supp}
\end{figure*}

\begin{figure*}[tb]
\includegraphics[width=.49\textwidth]{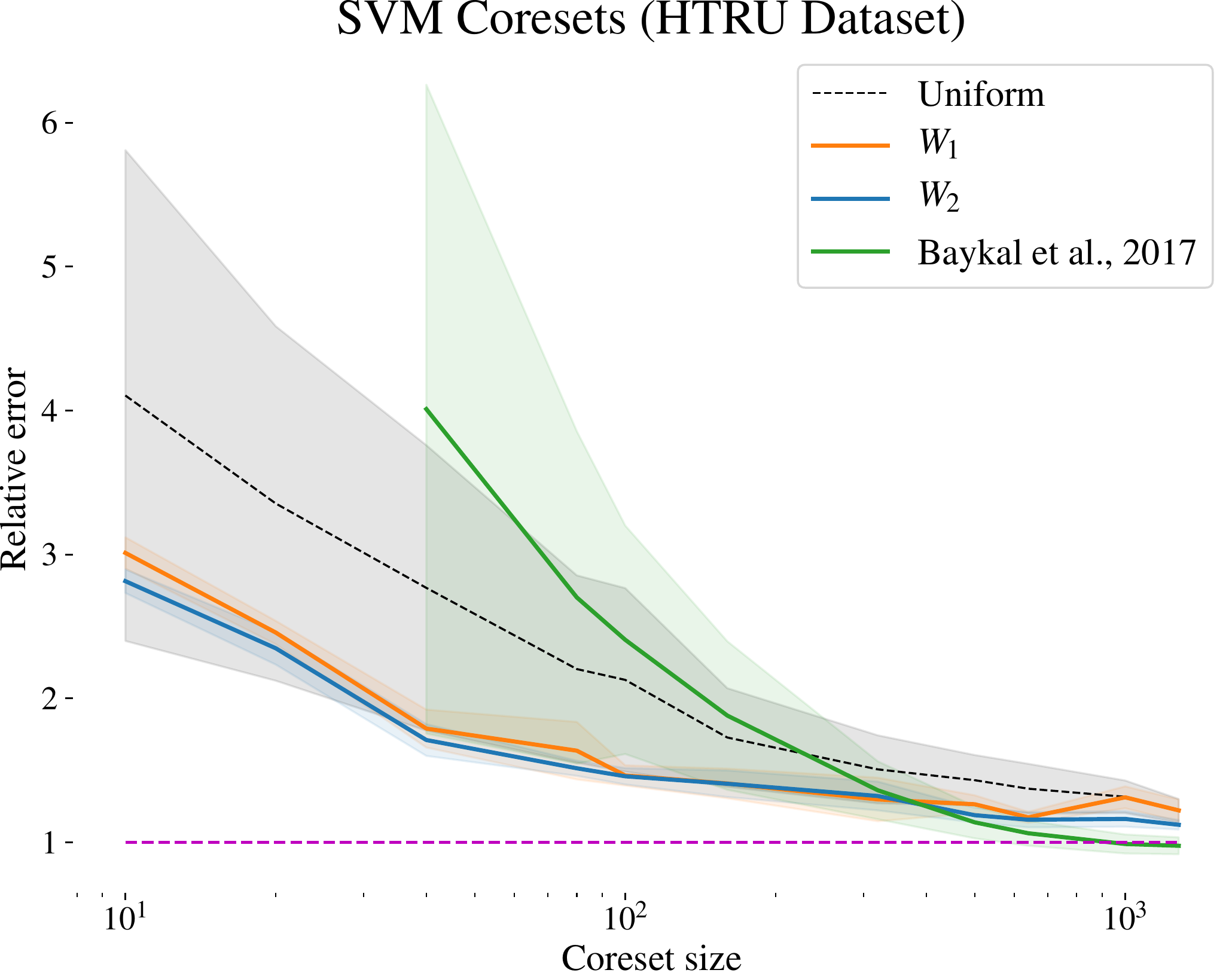}
\includegraphics[width=.49\textwidth]{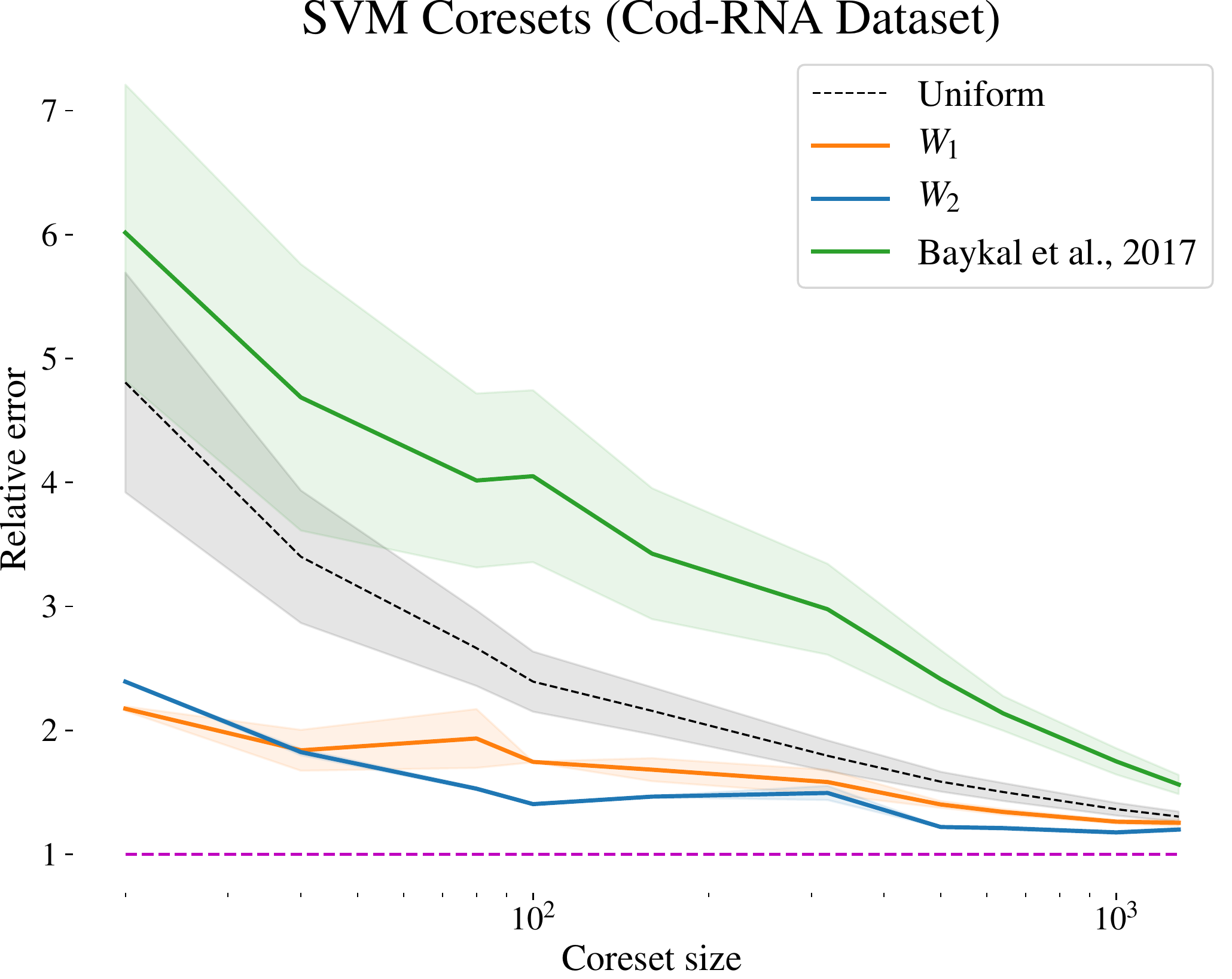}
\caption{Additional results for SVM classification.}
\label{fig:svm-supp}
\end{figure*}

\section{Additional Results}
We present additional experimental results on the HTRU dataset \cite{lyon2016fifty}, and the RNA coding dataset \cite{uzilov2006detection}. We test our SVM coreset and $k$-means coresets against uniform samples and state of the art coreset constructions.

Results for $k$-means are shown in Figure \ref{fig:kmeans-supp}. Results for SVMs are shown in Figure \ref{fig:svm-supp}.

\section{Comparison with Kernel Herding}
We have mentioned constructing coresets under the maximum mean discrepancy. Coresets under the MMD distance can be constructed using kernel herding, as shown in \cite{DBLP:conf/uai/ChenWS10,DBLP:conf/aistats/Lacoste-JulienL15}. We give a qualitative comparison between $W_2$ coresets and samples obtained from herding on the mixture of Gaussian example from \cite{DBLP:conf/uai/ChenWS10} in Figure \ref{fig:herding}.

\end{document}